\documentclass[11pt]{article}

\usepackage[T1]{fontenc}
\usepackage{lmodern}
\usepackage[margin=1in]{geometry}
\usepackage{microtype}

\usepackage{amsmath,amssymb,amsthm,mathtools}

\usepackage{graphicx}
\graphicspath{{figures/}}
\usepackage{booktabs}
\usepackage{array}
\usepackage{tabularx}

\usepackage[shortlabels]{enumitem}
\usepackage[numbers,sort&compress]{natbib}
\usepackage{xcolor}
\usepackage[colorlinks=true,linkcolor=blue,citecolor=blue,urlcolor=blue]{hyperref}
\usepackage[capitalise,noabbrev]{cleveref}

\theoremstyle{plain}
\newtheorem{theorem}{Theorem}[section]

\newtheorem{proposition}[theorem]{Proposition}
\newtheorem{corollary}[theorem]{Corollary}

\theoremstyle{definition}

\theoremstyle{remark}

\newcommand{\R}{\mathbb{R}}

\newcommand{\E}{\mathbb{E}}
\newcommand{\Pp}{\mathbb{P}}
\newcommand{\cF}{\mathcal{F}}
\newcommand{\cG}{\mathcal{G}}
\newcommand{\cH}{\mathcal{H}}

\newcommand{\ip}[2]{\left\langle #1,#2\right\rangle}
\newcommand{\norm}[1]{\left\lVert #1\right\rVert}
\newcommand{\pos}[1]{\left(#1\right)_{+}}
\DeclareMathOperator{\dist}{dist}
\DeclareMathOperator{\Reg}{Reg}
\DeclareMathOperator{\MMD}{MMD}

\renewcommand{\arraystretch}{1.15}

\title{From Approachability Residuals to Anytime-Valid Evidence:\\
The Online Convex Geometry of Testing by Betting}
\author{Jinze Zhao\\University of California, San Diego\\\texttt{jiz419@ucsd.edu}}
\date{}

\begin{document}

\maketitle

\begin{abstract}
Betting-based sequential tests and Blackwell approachability are linked by a
rate-explicit reduction through support-function residuals.  For a compact
convex target $S$ and vector observations $r_t$, an OCO learner selects a
predictable normal $w_t$ and produces
$q_t=\langle w_t,r_t\rangle-h_S(w_t)$.  We prove the exact pathwise identity
\[
  \dist(\bar r_T,S)
  =\frac1T\sum_{t=1}^Tq_t+\frac{\Reg_T}{T}.
\]
When $|q_t|\leq B$, composing this identity with one-sided betting yields a
finite-time transfer: if the OCO and log-wealth regrets are at most $a_T$ and
$\ell_T$, respectively, then a target gap exceeding
\[
  \frac{a_T}{T}
  +2B\sqrt{\frac{\log(1/\alpha)+\ell_T}{T}}
\]
forces rejection by time $T$, while non-rejection certifies the converse
radius.  We then formulate a controlled stochastic experiment in which an
action selected after $w_t$ satisfies Blackwell's supporting-halfspace
condition for every null mean payoff.  The resulting wealth is an e-process
under adaptive nulls; sublinear OCO regret gives stochastic approachability,
whereas persistent mean separation under an alternative gives exponential
wealth at rate at least $\delta^2/(4B^2)$.  Deterministic Blackwell games and
passive tests are, respectively, the noise-free and singleton-action cases of
this protocol.  Bounded two-sample means, kernel MMD, and active heterogeneous
data sources instantiate the reduction.  The resulting connection is exact
algebraically, quantitative at finite time, and operational when experiments
are controlled.

\end{abstract}


\section{Introduction}
\label{sec:introduction}

Sequential two-sample testing asks whether two data streams have the same law
while allowing the sample size to be chosen online.  The testing-by-betting
framework addresses optional stopping by maintaining nonnegative wealth: the
null is rejected once the wealth crosses $1/\alpha$, and Ville's inequality
controls the probability of ever crossing this threshold at level $\alpha$
\citep{ville1939,shafer2021}.  In the nonparametric construction of
\citet{shekhar2024}, a prediction algorithm chooses a bounded witness function
$g_t$ from past data, the next pair $(X_t,Y_t)$ produces the scalar payoff
\[
  v_t=g_t(X_t)-g_t(Y_t),
\]
and a second online algorithm chooses a betting fraction $\lambda_t$ in the
wealth update $K_t=K_{t-1}(1+\lambda_tv_t)$.  The statistical power of the test
is controlled by the regret of the witness learner, whereas anytime validity
follows from conditional fairness under the null.  A 2025 correction supplies
a repaired proof of the expected-stopping-time claim without changing its
order \citep{shekharcorr2025}.

The same ingredients---a predictable direction, a scalar payoff, and online
regret---drive Blackwell's approachability theory.  In a repeated vector-payoff
game, an agent selects actions so that the average payoff approaches a convex
target $S$ against every strategy of Nature \citep{blackwell1956}.
\citet{abernethy2011} established algorithmic reductions between
approachability and online linear optimization, and \citet{shimkin2016} gave a
direct OCO formulation for general convex targets using the support identity
\begin{equation}
  \dist(r,S)=\max_{\norm{w}\leq1}
  \{\ip{w}{r}-h_S(w)\}.
  \label{eq:support-distance-intro}
\end{equation}
In this formulation, OCO learns a normal direction $w_t$ and the Blackwell
action makes the signed support residual
$\ip{w_t}{r_t}-h_S(w_t)$ nonpositive.

This paper shows that the same residual organizes betting-based tests.  In a
passive test, the action set is a singleton and the halfspace condition holds
in conditional mean under the null.  Betting then converts positive departures
from that condition into anytime-valid evidence.  When the statistician can
choose a source, query, treatment, or experiment, the action set is nontrivial:
the selected action can enforce Blackwell's halfspace condition over a
composite null, while an active design policy preserves separation under
alternatives.  Thus the connection has an exact passive form and a literal
controlled-game form.

\paragraph{Contributions.}
\begin{enumerate}[(i)]
  \item We derive an exact residual--regret identity for any compact convex
        target.  The average support residual produced by an OCO normal learner
        equals distance to the target minus regret per round.
  \item We compose that identity with a one-sided test supermartingale and
        prove a finite-time rate-transfer theorem.  The Blackwell/OCO radius
        $a_T/T$ appears unchanged, and anytime evidence adds the explicit
        radius
        $2B\sqrt{\frac{\log(1/\alpha)+\ell_T}{T}}$.
  \item We recover bounded-mean and kernel-MMD betting tests by taking
        $S=\{0\}$ and embedding observations as half feature differences.  This
        makes the normalization, witness regret, and target gap explicit.
  \item We construct a controlled stochastic Blackwell--betting game.  A
        null-safe experiment response yields both an e-process and stochastic
        approachability under adaptive nulls; separated alternatives yield
        exponential wealth.  Deterministic approachability and passive testing
        are special cases.
  \item For heterogeneous two-sample sources, a direct-sum RKHS embedding
        prevents cross-source cancellation and gives an explicit evidence
        exponent in terms of source frequencies and MMD discrepancies.
\end{enumerate}

The reduction is exact at three levels: algebraically through the
residual--regret identity, quantitatively through the finite-time evidence
radius, and operationally through the controlled mean-payoff game.  This
structure also identifies where rate-sensitive concentration enters: the
pathwise reduction controls the empirical target gap, while sharp expected
stopping times require additional variance and tail analysis.

\section{Two Online Protocols}
\label{sec:background}

This section records the two protocols before identifying their common convex-analytic component.  Keeping the order of moves explicit is essential.

\subsection{Betting-based two-sample testing}

Let $(X_t,Y_t)_{t\geq1}$ be adapted to a filtration $(\cF_t)_{t\geq0}$.  In the i.i.d. two-sample problem, $(X_t,Y_t)\sim P_X\times P_Y$, and the null is $P_X=P_Y$.  Fix a symmetric class of functions $\cG\subset[-1/2,1/2]^{\mathcal X}$.  At round $t$:
\begin{enumerate}[(1)]
  \item a prediction algorithm chooses an $\cF_{t-1}$-measurable witness $g_t\in\cG$;
  \item a betting algorithm chooses an $\cF_{t-1}$-measurable fraction $\lambda_t$;
  \item the pair $(X_t,Y_t)$ is observed and $v_t=g_t(X_t)-g_t(Y_t)\in[-1,1]$;
  \item wealth is updated as $K_t=K_{t-1}(1+\lambda_tv_t)$.
\end{enumerate}
Under the equality null,
\begin{equation}
  \E[v_t\mid\cF_{t-1}]
  =\E[g_t(X_t)-g_t(Y_t)\mid\cF_{t-1}]=0.
  \label{eq:null-fairness}
\end{equation}
Consequently, predictable bets $\lambda_t\in[-1/2,1/2]$ make $(K_t)$ a nonnegative martingale.  Ville's inequality gives
\begin{equation}
  \sup_{P_X=P_Y}\Pp\!\left(\sup_{t\geq1}K_t\geq\frac1\alpha\right)\leq\alpha.
  \label{eq:ville}
\end{equation}

The witness learner is evaluated by its gain regret
\begin{equation}
  R_T
  =\sup_{g\in\cG}\sum_{t=1}^T\{g(X_t)-g(Y_t)\}
   -\sum_{t=1}^T\{g_t(X_t)-g_t(Y_t)\}.
  \label{eq:witness-regret}
\end{equation}
For an integral probability metric (IPM), the first term is an empirical support function.  If $R_T=o(T)$ and the population IPM is positive, the average realized payoff is eventually positive.  The Online Newton Step (ONS) betting update used by \citet{shekhar2024}, based on \citet{cutkosky2018}, then turns this margin into exponential wealth.  Their deterministic bound includes
\begin{equation}
  \log K_T\geq \frac{T}{8}\left(\frac1T\sum_{t=1}^Tv_t\right)^2-\log T,
  \label{eq:shekhar-wealth}
\end{equation}
up to the initialization conventions stated in that work.

\subsection{Blackwell approachability through OCO}

In the classical finite-dimensional game, an agent has actions $i\in I$, Nature has actions $j\in J$, and $r(i,j)\in\R^d$ is a vector payoff.  At round $t$, the agent chooses a mixed action $x_t\in\Delta(I)$, Nature chooses $j_t$, and the smoothed payoff is $r_t=r(x_t,j_t)$.  A closed convex set $S\subset\R^d$ is approachable if the agent can guarantee $\dist(\bar r_T,S)\to0$ against every strategy of Nature \citep{blackwell1956}.

For a compact convex $S$, its support function is
\[
  h_S(w)=\sup_{s\in S}\ip{w}{s}.
\]
Blackwell's separating-halfspace condition can be written as follows: for every $w$ in the Euclidean unit ball, the agent can choose a mixed action $x_S(w)$ satisfying
\begin{equation}
  \ip{w}{r(x_S(w),j)}-h_S(w)\leq0
  \quad\text{for every }j\in J.
  \label{eq:blackwell-response}
\end{equation}

\citet{shimkin2016} applies OCO to the convex losses
\begin{equation}
  f_t(w)=h_S(w)-\ip{w}{r_t},\qquad \norm{w}\leq1.
  \label{eq:shimkin-loss}
\end{equation}
The OCO learner outputs $w_t$ before $r_t$ is revealed; the agent then plays $x_t=x_S(w_t)$ so that \cref{eq:blackwell-response} holds.  If the OCO regret is at most $a(T)$, the distance is at most $a(T)/T$.  Online gradient descent gives the usual $O(T^{-1/2})$ rate, while follow-the-leader recovers Blackwell's projection direction.  Under a smooth-boundary assumption, Shimkin obtains a faster $O(\log T/T)$ bound; regularized follow-the-leader recovers the general $O(T^{-1/2})$ result.

The roles are therefore:
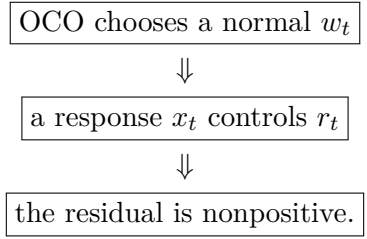
\[
  \begin{gathered}
    \boxed{\text{OCO chooses a normal }w_t}\\[-1pt]
    \Downarrow\\[-1pt]
    \boxed{\text{a response }x_t\text{ controls }r_t}\\[-1pt]
    \Downarrow\\[-1pt]
    \boxed{\text{the residual is nonpositive}.}
  \end{gathered}
\]
Only the first and last boxes have immediate analogues in a passive two-sample test.

\subsection{The shared support-function residual}

Both protocols scalarize a vector $r_t$ by a predictable direction $w_t$.  For a target $S$, define the signed residual
\begin{equation}
  q_t(w)=\ip{w}{r_t}-h_S(w).
  \label{eq:residual}
\end{equation}
It is positive when the direction $w$ separates $r_t$ from the supporting halfspace of $S$.  Blackwell's response makes $q_t(w_t)\leq0$.  A betting test instead treats a positive average residual as evidence that the conditional mean lies outside $S$.  The next section makes this statement exact.

\section{From an Approachability Residual to a Test Supermartingale}
\label{sec:bridge}

We work in a separable real Hilbert space $(\cH,\ip{\cdot}{\cdot})$.  Let $S\subset\cH$ be nonempty, compact, and convex, and let $\mathbb B=\{w:\norm{w}\leq1\}$.  The projection theorem and convex duality give
\begin{equation}
  \dist(r,S)=\sup_{w\in\mathbb B}\{\ip{w}{r}-h_S(w)\}.
  \label{eq:hilbert-distance}
\end{equation}
In the finite-dimensional case this is exactly the identity used by \citet{shimkin2016}.

At time $t$, an OCO algorithm selects an $\cF_{t-1}$-measurable $w_t\in\mathbb B$, after which $r_t$ is revealed.  Define the convex loss $f_t(w)=h_S(w)-\ip{w}{r_t}$ and realized regret
\begin{equation}
  \Reg_T=\sum_{t=1}^Tf_t(w_t)-\inf_{w\in\mathbb B}\sum_{t=1}^Tf_t(w).
  \label{eq:oco-regret}
\end{equation}

\begin{proposition}[Residual--regret identity]
\label{prop:identity}
For every deterministic sequence $r_1,\ldots,r_T$ and every sequence $w_1,\ldots,w_T\in\mathbb B$,
\begin{equation}
  \dist(\bar r_T,S)
  =\frac1T\sum_{t=1}^T\bigl(\ip{w_t}{r_t}-h_S(w_t)\bigr)
   +\frac{\Reg_T}{T}.
  \label{eq:exact-identity}
\end{equation}
\end{proposition}

\begin{proof}
Positive homogeneity of the support function implies
\[
  \inf_{w\in\mathbb B}\sum_{t=1}^Tf_t(w)
  =-T\sup_{w\in\mathbb B}\{\ip{w}{\bar r_T}-h_S(w)\}
  =-T\dist(\bar r_T,S).
\]
Also, $\sum_tf_t(w_t)=-\sum_t(\ip{w_t}{r_t}-h_S(w_t))$.  Substitution into \cref{eq:oco-regret} and rearrangement prove the result.
\end{proof}

This equality is the common mathematical skeleton.  Its two uses have opposite signs:
\begin{itemize}
  \item If a Blackwell response guarantees $\ip{w_t}{r_t}-h_S(w_t)\leq0$ in every round, then $\dist(\bar r_T,S)\leq\Reg_T/T$.
  \item If the environment produces $\dist(\bar r_T,S)$ bounded away from zero and $\Reg_T=o(T)$, then the average residual is positive and can be bet on.
\end{itemize}

\subsection{A general approachability-gap monitor}

Assume there is a known $B>0$ such that
\begin{equation}
  \left|\ip{w}{r_t}-h_S(w)\right|\leq B
  \quad\text{for all }w\in\mathbb B\text{ and all }t.
  \label{eq:bounded-residual}
\end{equation}
Set
\[
  v_t=\frac{\ip{w_t}{r_t}-h_S(w_t)}{B}\in[-1,1].
\]
Let a betting algorithm choose predictable $\lambda_t\in[0,1/2]$ and update $K_t=K_{t-1}(1+\lambda_tv_t)$ from $K_0=1$.  We use nonnegative bets because the general convex null gives a supermartingale rather than necessarily a martingale.

\begin{theorem}[Anytime-valid approachability-gap monitor]
\label{thm:gap-monitor}
Suppose the null hypothesis is
\begin{equation}
  H_0:\quad m_t:=\E[r_t\mid\cF_{t-1}]\in S
  \quad\text{almost surely for every }t.
  \label{eq:convex-null}
\end{equation}
Then $(K_t)$ is a nonnegative supermartingale, and the stopping rule
\[
  \tau_\alpha=\inf\{t\geq1:K_t\geq1/\alpha\}
\]
satisfies $\sup_{H_0}\Pp(\tau_\alpha<\infty)\leq\alpha$.

Suppose in addition that the betting algorithm has log-wealth regret $L_T$ relative to the best constant bet in $[0,1/2]$:
\begin{equation}
  \log K_T\geq
  \max_{0\leq\lambda\leq1/2}\sum_{t=1}^T\log(1+\lambda v_t)-L_T.
  \label{eq:bet-regret}
\end{equation}
Then, pathwise,
\begin{equation}
  \log K_T\geq
  \frac{T}{4B^2}
  \pos{\dist(\bar r_T,S)-\Reg_T/T}^{2}-L_T.
  \label{eq:wealth-distance}
\end{equation}
In particular, projected ONS on $[0,1/2]$ has $L_T=O(\log T)$.
\end{theorem}

\begin{proof}
Because $w_t$ is predictable and $m_t\in S$,
\[
  \E[v_t\mid\cF_{t-1}]
  =\frac{\ip{w_t}{m_t}-h_S(w_t)}{B}\leq0.
\]
Thus $\E[K_t\mid\cF_{t-1}]\leq K_{t-1}$, and nonnegativity follows from $\lambda_t\in[0,1/2]$ and $v_t\geq-1$.  Ville's inequality proves validity.  The elementary comparator bound in \cref{app:constant-bet} gives
\[
  \max_{0\leq\lambda\leq1/2}\sum_{t=1}^T\log(1+\lambda v_t)
  \geq \frac{T}{4}\pos{\bar v_T}^{2}.
\]
Finally, \cref{prop:identity} gives $B\bar v_T=\dist(\bar r_T,S)-\Reg_T/T$.
\end{proof}

\begin{corollary}[Power under a separated stationary alternative]
\label{cor:power}
Suppose $(r_t)$ are i.i.d., bounded, and have mean $\mu$ with $\delta=\dist(\mu,S)>0$.  If $\Reg_T=o(T)$ almost surely and $L_T=o(T)$, then
\[
  \liminf_{T\to\infty}\frac1T\log K_T\geq\frac{\delta^2}{4B^2}
  \quad\text{almost surely}.
\]
Consequently $\Pp(\tau_\alpha<\infty)=1$ for every $\alpha\in(0,1)$.
\end{corollary}

\begin{proof}
The Hilbert-space strong law and continuity of distance imply $\dist(\bar r_T,S)\to\delta$.  Apply \cref{eq:wealth-distance} and the two sublinear-regret assumptions.
\end{proof}

The theorem is deliberately modular.  The OCO layer estimates a support-function gap.  The betting layer supplies optional-stopping validity and amplifies a positive gap.  Neither layer alone provides both properties.

\section{Two-Sample Specializations}
\label{sec:examples}

We now take $S=\{0\}$.  Then $h_S\equiv0$, and the residual is simply $\ip{w_t}{r_t}$.  Under equality of the two distributions, the conditional mean of $r_t$ is zero, so two-sided bets are also valid and one recovers the exact martingale construction of \citet{shekhar2024}.  The one-sided formulation in \cref{thm:gap-monitor} is sufficient for power because a symmetric witness class orients the population discrepancy positively.

\subsection{Bounded two-sample mean testing}

Assume $X_t,Y_t\in\R^m$ with $\norm{X_t}_2,\norm{Y_t}_2\leq1$, and define
\begin{equation}
  r_t=\frac{X_t-Y_t}{2},\qquad S=\{0\},\qquad \norm{w_t}_2\leq1.
  \label{eq:bounded-embedding}
\end{equation}
The scalar residual $v_t=\ip{w_t}{r_t}$ belongs to $[-1,1]$, and
\begin{equation}
  \dist(\E r_t,S)=\frac12\norm{\E X-\E Y}_2=: \Delta.
  \label{eq:mean-delta}
\end{equation}
The witness class in \citet{shekhar2024} is $g_u(x)=\ip{u}{x}$ with $\norm{u}_2\leq1/2$.  The reparameterization $w=2u$ makes its payoff
\[
  g_u(X_t)-g_u(Y_t)=\ip{w}{r_t}.
\]
Thus the adaptive online-gradient-ascent update in that paper is exactly a no-regret normal-direction learner for the target $\{0\}$, after this harmless rescaling.

For the gain regret
\[
  R_T=\sup_{\norm{w}\leq1}\sum_{t=1}^T\ip{w}{r_t}
      -\sum_{t=1}^T\ip{w_t}{r_t},
\]
\cref{prop:identity} becomes
\begin{equation}
  \frac1T\sum_{t=1}^Tv_t
  =\norm{\bar r_T}_2-\frac{R_T}{T}.
  \label{eq:mean-identity}
\end{equation}
The empirical Euclidean mean gap is therefore not merely analogous to the learner's average payoff: the two are equal up to regret per round.  Under the alternative, $\norm{\bar r_T}\to\Delta$ and adaptive OGA has $R_T=O(\sqrt{\sum_t\norm{r_t}^2})$, so the bettor eventually receives a positive average payoff.

Combining OGA with the sharper variance-sensitive analysis in the corrected betting paper yields
\begin{equation}
  \E[\tau_\alpha]
  =O\!\left(
    \frac{\sigma^2\log(1/(\alpha\Delta))}{\Delta^2}
    +\frac{\log(1/(\alpha\Delta))}{\Delta}
  \right),
  \label{eq:corrected-stopping}
\end{equation}
where $\sigma^2$ is the supremum of the payoff variance over the witness ball \citep{shekhar2024,shekharcorr2025}.  The deterministic residual identity explains the appearance of witness regret and $\Delta$; it does not by itself prove the summability needed for \cref{eq:corrected-stopping}.

\subsection{Sequential kernel MMD}

Let $k$ be a positive-definite kernel with $k(x,x)\leq1$, let $\cH_k$ be its reproducing kernel Hilbert space, and let $\phi(x)=k(x,\cdot)$.  Define
\begin{equation}
  r_t=\frac{\phi(X_t)-\phi(Y_t)}{2}\in\cH_k,qquad
  S=\{0\},\qquad \norm{w_t}_{\cH_k}\leq1.
  \label{eq:mmd-embedding}
\end{equation}
Because $\norm{\phi(x)}_{\cH_k}\leq1$, we have $\norm{r_t}_{\cH_k}\leq1$.  The population separation is
\begin{equation}
  \dist(\E r_t,S)
  =\frac12\norm{\mu_{P_X}-\mu_{P_Y}}_{\cH_k}
  =\frac12\MMD_k(P_X,P_Y).
  \label{eq:mmd-distance}
\end{equation}
For characteristic kernels this is positive exactly when $P_X\neq P_Y$ \citep{gretton2012,sriperumbudur2011}.

Set $g_t=w_t/2$.  The residual equals
\begin{equation}
  v_t=\ip{w_t}{r_t}_{\cH_k}
  =g_t(X_t)-g_t(Y_t),
  \label{eq:mmd-payoff}
\end{equation}
which is precisely the payoff used in the sequential kernel-MMD test.  Its OGA witness update is an OCO learner on linear losses in $\cH_k$.  The residual--regret identity reads
\begin{equation}
  \frac1T\sum_{t=1}^Tv_t
  =\frac12\MMD_k(\widehat P_{X,T},\widehat P_{Y,T})-\frac{R_T}{T},
  \label{eq:mmd-identity}
\end{equation}
where the displayed empirical MMD is the norm of the difference between empirical mean embeddings.  This equality formalizes the sense in which no-regret witness learning tracks a maximum discrepancy.

The resulting test has three logically separate guarantees:
\begin{enumerate}[(a)]
  \item conditional fairness of \cref{eq:mmd-payoff} under $P_X=P_Y$ gives anytime validity;
  \item sublinear witness regret and characteristicness give a positive limiting margin and hence power one;
  \item ONS wealth bounds plus variance-sensitive concentration give the refined stopping-time and type-II exponent results of \citet{shekhar2024}.
\end{enumerate}

\subsection{A convex mean-null beyond equality testing}

The same reduction is not limited to $S=\{0\}$.  Suppose $r_t$ is a vector statistic and the null asserts $\E[r_t\mid\cF_{t-1}]\in S$ for a known compact convex set.  Examples include moment-inequality nulls, multivariate tolerance regions, and conditional calibration constraints.  The support-function residual $q_t=\ip{w_t}{r_t}-h_S(w_t)$ is conditionally superfair under the null.  Therefore \cref{thm:gap-monitor} yields an anytime-valid test of the entire convex null, while OCO adaptively searches for a violated supporting halfspace.  This is the natural testing counterpart of Shimkin's general-target formulation.

\section{Blackwell--Betting Reductions}
\label{sec:discussion}

The common residual supports three increasingly strong statements.  First, the
residual--regret identity is an exact algebraic equivalence between a vector
target gap and cumulative scalar gain.  Second, betting converts that equality
into a finite-time separation theorem: a target gap outside an explicit radius
forces rejection.  Third, when the statistician controls the experiment or
data source, the action can satisfy Blackwell's supporting-halfspace condition
for the null mean payoff.  The resulting protocol is a stochastic Blackwell
game equipped with an anytime-valid evidence process.  This section develops
these three statements and places them in the surrounding literature.

\subsection{Three converging lines of work}

The first line begins with Blackwell's halfspace characterization of
approachability \citep{blackwell1956}.  The algorithmic equivalence between
approachability and online linear optimization was established by
\citet{abernethy2011}; \citet{perchet2014} develops the broader network of
equivalences among approachability, regret, and calibration.  Two later
directions are especially relevant here.  The response-based method of
\citet{bernstein2015} replaces projection by a feasible response to a forecast
of Nature's mixed action, whereas \citet{shimkin2016} uses the support function
to obtain a direct OCO algorithm for an arbitrary convex target.  These are
distinct constructions: the former uses Blackwell's dual response condition,
while the latter uses the primal supporting-halfspace condition in
\cref{eq:blackwell-response}.  Unknown games and partial monitoring show how
the available feedback and knowledge of the payoff map alter what can be
approached and at what rate \citep{mannor2014,kwon2017}.

The progression also locates the assumptions behind each reduction.
Abernethy, Bartlett, and Hazan use conic duality to pass between online linear
optimization and approachability.  Shimkin's support-function loss avoids a
conic lift and exposes the residual used here directly.  By contrast, the
unknown-game results of \citet{mannor2014} show that a best target selected in
hindsight is generally unattainable without relaxation.  The optimal
$T^{-1/2}$ and $T^{-1/3}$ partial-monitoring rates of \citet{kwon2017} further
show that the observation channel is part of the statistical problem, not a
cosmetic modification of the full-information game.

The second line treats wealth maximization itself as online learning.  Test
martingales and e-processes convert conditional fairness into optional-stopping
validity \citep{ville1939,shafer2021}.  Coin-betting potentials, in the reverse
direction, generate parameter-free online linear optimization algorithms
\citep{orabona2016}, and scalar betting primitives support black-box online
learning reductions in general normed spaces \citep{cutkosky2018}.  In the
two-sample test of \citet{shekhar2024}, these ideas appear as two online
learners: a witness learner maximizes a predictable IPM payoff, and an ONS
bettor compounds that payoff into evidence.

The direction of the coin-betting reduction is complementary to ours.
\citet{orabona2016} start from a wealth guarantee and recover online-learning
regret through convex conjugacy.  Here the normal learner first produces an
OCO-certified residual, after which scalar wealth turns that residual into a
test.  The two reductions therefore meet at $q_t$, even though the geometric
potential for approachability and the log-wealth potential remain different.

The third line is active sequential experimentation.  In classical parametric
models, the experiment is selected adaptively to increase the information rate
\citep{chernoff1959,naghshvar2013}.  The same control variable now appears in
nonparametric testing: \citet{hsu2025} combine adaptive selection among
heterogeneous data sources with betting-based two-sample inference.  This
active setting supplies exactly the object required for a full Blackwell
protocol---an action chosen after the normal direction and before the next
observation.

Classical active-testing policies optimize model-based information rates.
The nonparametric procedure of \citet{hsu2025} instead combines exploitation
of empirically distinguishable sources with continuing exploration, and proves
level-$\alpha$ validity, power one, and an expected stopping-time bound.  The
controlled construction in \cref{sec:controlled-game} isolates the geometric
condition beneath both settings: action selection must preserve a null
supporting halfspace, while its alternative objective is to maintain positive
distance from the target.

These literatures do not imply that arbitrary black-box reductions preserve
optimal constants or rates.  Indeed, \citet{dann2025} show that standard
approachability--regret reductions can distort the optimal rate.  The results
below avoid such an inheritance claim: their rates follow directly from the
exact identity \cref{eq:exact-identity} for the particular support-function
instance used by the test.

\subsection{An exact protocol dictionary}

\Cref{tab:dictionary} separates the objects that coincide from the feature that
changes across the three protocols.  The controlled model contains
deterministic approachability and passive testing as boundary cases.

The dictionary yields three levels of equivalence.  At the \emph{algebraic}
level, \cref{eq:exact-identity} is an equality on every payoff sequence.  At
the \emph{quantitative} level, \cref{thm:finite-rate-transfer} below maps the
Blackwell/OCO radius into an anytime evidence radius with an explicit extra
term.  At the \emph{operational} level,
\cref{thm:controlled-blackwell-betting} places the direction, response action,
random payoff, and wealth update in one protocol.

\begin{center}
  \refstepcounter{table}\label{tab:dictionary}
  \begin{minipage}{\textwidth}
  \small
  \noindent Table~\thetable: The Blackwell--betting dictionary.  The same
  support residual is controlled pathwise, merely observed, or controlled in
  conditional mean.\par\smallskip
  \centering
  \renewcommand{\arraystretch}{1.0}
  \begin{tabularx}{\linewidth}{@{}>{\raggedright\arraybackslash}p{0.14\textwidth}>{\raggedright\arraybackslash}X>{\raggedright\arraybackslash}X>{\raggedright\arraybackslash}X@{}}
    \toprule
    Object & Deterministic Blackwell game & Passive betting test & Controlled stochastic reduction \\
    \midrule
    Vector payoff
      & $r(x_t,j_t)$
      & data statistic $r_t$
      & $R(a_t,Z_t)$ from the selected experiment \\
    Normal direction
      & OCO output $w_t$
      & learned witness $w_t$
      & OCO output $w_t$ \\
    Response action
      & $x_t=x_S(w_t)$
      & singleton action set
      & $a_t\in\mathcal A_0(w_t)$ \\
    Residual
      & $q_t=\ip{w_t}{r_t}-h_S(w_t)$
      & the same $q_t$
      & the same $q_t$ \\
    Null control
      & $q_t\leq0$ for every move
      & $\E[q_t\mid\cF_{t-1}]\leq0$
      & the action enforces $\E[q_t\mid\cF_{t-1}]\leq0$ \\
    Output
      & deterministic target radius
      & anytime-valid evidence
      & target convergence under the null and evidence under alternatives \\
    \bottomrule
  \end{tabularx}
  \end{minipage}
\end{center}

\subsection{Passive testing as the singleton-action game}

A passive test is a stochastic mean-payoff game with the singleton action set
$\mathcal A=\{a_0\}$.  The null condition
$\E[r_t\mid\cF_{t-1}]\in S$ is equivalent, by support-function separation, to
\begin{equation}
  \ip{w}{\E[r_t\mid\cF_{t-1}]}-h_S(w)\leq0
  \quad\text{for every }w\in\mathbb B.
  \label{eq:singleton-blackwell}
\end{equation}
Thus the passive protocol already satisfies the stochastic halfspace
condition; its response map is simply degenerate.  Under the null, the average
payoff approaches $S$ through a martingale law of large numbers.  Under an
alternative, the same normal learner estimates a supporting direction of the
empirical target violation, and wealth records the violation without losing
validity under optional stopping.

There are also two different optimization variables, each with a precise role.
The vector $w_t$ is the geometric dual variable in the support representation
of distance.  The scalar $\lambda_t$ is a capital fraction optimized for log
wealth.  The composition is therefore
\begin{equation}
  \underbrace{(r_t,S)}_{\text{vector target geometry}}
  \ \xrightarrow{\ \text{OCO in }w\ }\
  \underbrace{q_t}_{\text{Blackwell residual}}
  \ \xrightarrow{\ \text{online betting in }\lambda\ }\
  \underbrace{K_t}_{\text{anytime evidence}}.
  \label{eq:three-layer-reduction}
\end{equation}
This is a composition of two online reductions, rather than a Lagrangian
primal--dual pair: $\lambda_t$ does not replace the experimental response
action.

\subsection{A finite-time rate-transfer theorem}
\label{sec:rate-transfer}

The residual identity transfers finite-time guarantees, not only asymptotic
consistency.  Let
\[
  D_T=\dist(\bar r_T,S)
\]
and suppose that the direction learner and bettor admit simultaneous pathwise
envelopes for nonnegative sequences $(a_T)$ and $(\ell_T)$,
\begin{equation}
  \Reg_T\leq a_T,
  \qquad
  L_T\leq\ell_T,
  \qquad T\geq1.
  \label{eq:two-regret-envelopes}
\end{equation}
Define the evidence radius
\begin{equation}
  c_T(\alpha)
  =2B\sqrt{\frac{\log(1/\alpha)+\ell_T}{T}}.
  \label{eq:evidence-radius}
\end{equation}

\begin{theorem}[Finite-time Blackwell--betting rate transfer]
\label{thm:finite-rate-transfer}
Under the conditions of \cref{thm:gap-monitor} and
\cref{eq:two-regret-envelopes}, every data path satisfies
\begin{equation}
  \log K_T
  \geq
  \frac{T}{4B^2}
  \pos{D_T-\frac{a_T}{T}}^2-\ell_T.
  \label{eq:finite-rate-transfer}
\end{equation}
Consequently, for every $T$,
\begin{equation}
  D_T\geq\frac{a_T}{T}+c_T(\alpha)
  \quad\Longrightarrow\quad
  K_T\geq\frac1\alpha
  \quad\Longrightarrow\quad
  \tau_\alpha\leq T.
  \label{eq:crossing-boundary}
\end{equation}
Equivalently, non-rejection through time $T$ gives the pathwise certificate
\begin{equation}
  \tau_\alpha>T
  \quad\Longrightarrow\quad
  D_T<\frac{a_T}{T}+c_T(\alpha).
  \label{eq:nonrejection-certificate}
\end{equation}
Under the conditional-mean null \cref{eq:convex-null}, this certificate is
simultaneous in time with probability at least $1-\alpha$:
\begin{equation}
  \Pp_0\!\left(
    \forall T\geq1,\quad
    D_T<\frac{a_T}{T}+c_T(\alpha)
  \right)\geq1-\alpha.
  \label{eq:time-uniform-radius}
\end{equation}
\end{theorem}

\begin{proof}
Monotonicity of the positive-part function, the bounds
$\Reg_T\leq a_T$ and $L_T\leq\ell_T$, and
\cref{eq:wealth-distance} give \cref{eq:finite-rate-transfer}.  If the first
inequality in \cref{eq:crossing-boundary} holds, its right-hand side is at
least $\log(1/\alpha)$, hence $K_T\geq1/\alpha$.  Taking the contrapositive
gives \cref{eq:nonrejection-certificate}.  Finally, Ville's inequality shows
that the event $\sup_TK_T<1/\alpha$ has probability at least $1-\alpha$ under
the null; on that event the certificate holds for every $T$.
\end{proof}

The theorem gives a precise rate dictionary.  A classical all-outcomes
Blackwell response makes every $q_t\leq0$, so the exact identity gives the
deterministic radius $D_T\leq a_T/T$.  With conditional-mean control, the same
algorithmic radius $a_T/T$ is inherited without algebraic loss and betting adds
only $c_T(\alpha)$.  For $a_T=O(\sqrt T)$ and
$\ell_T=O(\log T)$, the evidence boundary is
\begin{equation}
  D_T=O\!\left(\sqrt{\frac{\log(T/\alpha)}{T}}\right).
  \label{eq:standard-evidence-rate}
\end{equation}
If the smooth-target FTL analysis of \citet{shimkin2016} gives
$a_T=O(\log T)$, the statistical evidence radius rather than the geometric
regret term becomes dominant.  In a partial-monitoring extension, an effective
approachability radius of order $T^{-1/3}$ can instead dominate the statistical
radius \citep{kwon2017}; the payoff estimator and feedback model must then be
included in the reduction.  For the full-information instance here, the rate
follows from an exact identity, so the non-rate-preservation phenomenon of
\citet{dann2025} does not introduce an additional hidden factor.

\begin{corollary}[Detection under finite separation]
\label{cor:finite-detection}
Fix $\delta>0$.  Suppose that on an event $E_T$,
\begin{equation}
  D_T\geq\delta-b_T.
  \label{eq:finite-separation-event}
\end{equation}
If
\begin{equation}
  b_T+\frac{a_T}{T}\leq\frac\delta2
  \quad\text{and}\quad
  \frac{T\delta^2}{16B^2}
  \geq\log(1/\alpha)+\ell_T,
  \label{eq:finite-detection-conditions}
\end{equation}
then $\tau_\alpha\leq T$ on $E_T$.  In particular, if
$\Pp_1(E_T)\geq1-\eta$, then
$\Pp_1(\tau_\alpha>T)\leq\eta$.
\end{corollary}

\begin{proof}
On $E_T$, \cref{eq:finite-separation-event,eq:finite-detection-conditions}
give $D_T-a_T/T\geq\delta/2$.  Therefore
\cref{eq:finite-rate-transfer} yields
\[
  \log K_T\geq\frac{T\delta^2}{16B^2}-\ell_T
  \geq\log(1/\alpha).
\]
\end{proof}

For example, if $a_T\leq c_R\sqrt T$ and
$\ell_T\leq c_L\log(T+1)$, the online-learning part of
\cref{eq:finite-detection-conditions} holds once
$T\geq16c_R^2/\delta^2$ and $b_T\leq\delta/4$; the remaining wealth condition
has the usual $B^2\delta^{-2}\log(1/\alpha)$ scale, with a logarithmic overhead
from $\ell_T$.  A concentration inequality supplies $b_T$ and therefore a
high-probability stopping bound.  Summable tail probabilities are the
additional ingredient needed to integrate that bound into an expected stopping
time, exactly the step repaired in \citet{shekharcorr2025}.

\subsection{Controlled experiments: a stochastic Blackwell--betting game}
\label{sec:controlled-game}

We now add a genuine control action.  Let $\mathcal A$ be a space of
experiments, queries, treatments, or data sources; randomized experiments may
be included as elements of $\mathcal A$.  Let $\mathcal P_0$ be a composite
null class.  Selecting $a\in\mathcal A$ and then observing
$Z\sim P(\cdot\mid a)$ produces a vector payoff
$R(a,Z)\in\cH$.  Write
\begin{equation}
  m_P(a)=\E_P[R(a,Z)\mid a]
  \label{eq:controlled-mean-payoff}
\end{equation}
for the mean-payoff map.

For each normal $w\in\mathbb B$, define the null-safe response set
\begin{equation}
  \mathcal A_0(w)
  =\left\{
    a\in\mathcal A:
    \sup_{P\in\mathcal P_0}
    \bigl\{\ip{w}{m_P(a)}-h_S(w)\bigr\}\leq0
  \right\}.
  \label{eq:null-safe-response-set}
\end{equation}
Assume that $\mathcal A_0(w)$ is nonempty for every $w\in\mathbb B$ and
admits a measurable selector.  This is Blackwell's primal
supporting-halfspace condition applied to the mean-payoff game.  The stochastic
Blackwell--betting protocol is
\begin{enumerate}[(1)]
  \item the OCO learner selects a predictable $w_t\in\mathbb B$;
  \item the controller selects a predictable
        $a_t\in\mathcal A_0(w_t)$;
  \item under the null, Nature selects an adapted conditional law
        $P_t(\cdot\mid a_t)\in\mathcal P_0$, and $Z_t$ is drawn from it;
  \item the payoff $r_t=R(a_t,Z_t)$ is observed, and the predictable bettor
        updates wealth using
        \[
          v_t=\frac{\ip{w_t}{r_t}-h_S(w_t)}{B}.
        \]
\end{enumerate}

\begin{theorem}[Controlled Blackwell--betting reduction]
\label{thm:controlled-blackwell-betting}
Suppose \cref{eq:null-safe-response-set} holds and the residuals satisfy
\cref{eq:bounded-residual}.  Under every adaptive null sequence
$P_t\in\mathcal P_0$:
\begin{enumerate}[(i)]
  \item $(K_t)$ is a nonnegative supermartingale, and hence
        $\Pp_0(\tau_\alpha<\infty)\leq\alpha$;
  \item if $\Reg_T=o(T)$ almost surely, then
        $\dist(\bar r_T,S)\to0$ almost surely.
\end{enumerate}

Under an alternative law, set
$m_t=\E[r_t\mid\cF_{t-1}]$.  Suppose
\begin{equation}
  \liminf_{T\to\infty}
  \dist\!\left(\frac1T\sum_{t=1}^Tm_t,S\right)
  \geq\delta>0
  \quad\text{almost surely},
  \label{eq:controlled-detectability}
\end{equation}
and suppose the Hilbert-valued martingale noise obeys
\begin{equation}
  \norm{\frac1T\sum_{t=1}^T(r_t-m_t)}\longrightarrow0
  \quad\text{almost surely}.
  \label{eq:vector-martingale-slln}
\end{equation}
If also $\Reg_T=o(T)$ and $L_T=o(T)$ almost surely, then
\begin{equation}
  \liminf_{T\to\infty}\frac1T\log K_T
  \geq\frac{\delta^2}{4B^2}
  \quad\text{almost surely}.
  \label{eq:controlled-exponent}
\end{equation}
Consequently, $\tau_\alpha<\infty$ almost surely for every
$\alpha\in(0,1)$.
\end{theorem}

\begin{proof}
Under the null, the response condition gives
\[
  \E[\ip{w_t}{r_t}-h_S(w_t)\mid\cF_{t-1}]
  =\ip{w_t}{m_{P_t}(a_t)}-h_S(w_t)\leq0.
\]
The proof of \cref{thm:gap-monitor} therefore establishes the
supermartingale and type-I error claims.  The residual minus its conditional
mean is a bounded scalar martingale difference, so its time average converges
to zero almost surely.  Hence the average residual has nonpositive limiting
superior.  Combining this fact with \cref{eq:exact-identity} and sublinear
regret forces $\dist(\bar r_T,S)\to0$.

Under the alternative, distance to a closed set is 1-Lipschitz.  Therefore
\cref{eq:controlled-detectability} and
\cref{eq:vector-martingale-slln} imply
$\liminf_T\dist(\bar r_T,S)\geq\delta$.  Apply
\cref{eq:wealth-distance} and the two sublinear-regret assumptions to obtain
\cref{eq:controlled-exponent}.
\end{proof}

The mean response is a precise stochastic relaxation of Blackwell's classical
condition.  If the stronger all-outcomes response
\begin{equation}
  \sup_z\bigl\{
    \ip{w}{R(a_S(w),z)}-h_S(w)
  \bigr\}\leq0
  \label{eq:all-outcomes-response}
\end{equation}
holds for every $w$, then each realized residual is nonpositive and the
protocol reduces to deterministic OCO approachability.  If
$\mathcal A=\{a_0\}$, it reduces to the passive conditional-mean test of
\cref{eq:singleton-blackwell}.  With a nontrivial $\mathcal A$, the controller
may choose among null-safe experiments to maximize alternative separation.
Thus the same residual governs deterministic approachability, stochastic
approachability, and sequential detection, with the mode of control determining
which guarantee is active.

\paragraph{Heterogeneous two-sample sources.}
Let $a\in\{1,\ldots,M\}$ index a data source.  At source $a$, observe
$(X^a,Y^a)\sim P_a\times Q_a$, let $k_a(x,x)\leq1$, and write
$\phi_a$ for its feature map in $\cH_a$.  Work in the direct-sum space
$\cH=\bigoplus_{a=1}^M\cH_a$ and define
\begin{equation}
  R\bigl(a,(x,y)\bigr)
  =\frac12 e_a\otimes\{\phi_a(x)-\phi_a(y)\}.
  \label{eq:direct-sum-source-payoff}
\end{equation}
Then $\norm{R(a,(x,y))}\leq1$, so $B=1$ for $S=\{0\}$.  Under the
sourcewise global null $P_a=Q_a$ for every $a$, every source is null-safe;
hence source selection can be fully adaptive without changing the e-process
validity.

Let $N_a(T)$ be the number of selections of source $a$ and
$\rho_{a,T}=N_a(T)/T$.  With
\begin{equation}
  \Delta_a
  =\norm{\mu_{P_a}-\mu_{Q_a}}_{\cH_a}
  =\MMD_{k_a}(P_a,Q_a),
  \label{eq:source-mmd}
\end{equation}
orthogonality of the direct-sum blocks gives the exact mean separation
\begin{equation}
  \norm{\frac1T\sum_{t=1}^Tm_t}^{2}
  =\frac14\sum_{a=1}^M\rho_{a,T}^{2}\Delta_a^2.
  \label{eq:source-separation}
\end{equation}
The embedding prevents discrepancies at different sources from cancelling.
If an informative source $a^\star$ has $\Delta_{a^\star}>0$ and
$\liminf_T\rho_{a^\star,T}\geq\rho>0$, then
\cref{thm:controlled-blackwell-betting} yields
\begin{equation}
  \liminf_{T\to\infty}\frac1T\log K_T
  \geq\frac{\rho^2\Delta_{a^\star}^{2}}{16}
  \quad\text{almost surely}.
  \label{eq:source-wealth-exponent}
\end{equation}
Forced exploration supplies the positive-frequency condition; an active
source learner can improve the exponent by allocating more mass to informative
sources.  This direct-sum construction embeds the active heterogeneous-source
problem of \citet{hsu2025} in the Blackwell--betting geometry and makes the
control-to-evidence link explicit.

\section{Conclusion}
\label{sec:conclusion}

Blackwell approachability and betting-based two-sample testing share an exact
online reduction.  A no-regret learner selects supporting normals, and the
resulting scalar residual differs from distance to the convex target by exactly
regret per round.  Betting converts that residual into evidence: outside the
Blackwell/OCO radius, an explicit $T^{-1/2}$ anytime boundary forces wealth to
cross the rejection threshold.  Conversely, non-rejection supplies a
time-uniform certificate that the empirical target gap remains within the sum
of the algorithmic and evidence radii.

The controlled mean-payoff construction makes the connection operational.  A
Blackwell response action enforces the supporting halfspace for every null mean
payoff; martingale noise is handled by the e-process.  Under adaptive nulls,
the same protocol gives stochastic approachability and type-I error control.
Under alternatives that remain separated under the selected experiments, it
gives exponential wealth and power one.  Deterministic Blackwell games arise
when the observation kernel is degenerate, while ordinary passive tests arise
when the action set is a singleton.

The two online variables retain distinct meanings within this unified game.
The witness $w_t$ is the geometric dual normal, and the bet $\lambda_t$ is a
one-dimensional capital control.  Their composition recovers bounded-mean and
kernel-MMD tests and extends to active heterogeneous sources through a
direct-sum feature representation.  The exact residual identity preserves the
OCO radius for this constructed instance; variance-sensitive concentration and
summable tail control remain the additional ingredients for sharp expected
stopping-time guarantees.

\section{Disclosure}
\label{sec:disclosure}
The proof strategy and counterexample were produced by OpenAI’s GPT-5.6 Sol Ultra through Codex in
response to prompts from the author. Codex was also used to revise
the exposition and prepare the LaTeX manuscript. The author selected the
problem, directed the interactions and revisions, and is the sole named author. The AI system is acknowledged as a reasoning and writing tool, not
as an author. This disclosure is not a substitute for independent expert
mathematical review.
\bibliographystyle{plainnat}
\bibliography{references}

\appendix
\section{Auxiliary Derivations}
\label{app:additional}

\subsection{A constant-bet lower bound}
\label{app:constant-bet}

For completeness, let $v_1,\ldots,v_T\in[-1,1]$ and $\bar v_T=T^{-1}\sum_tv_t$.  For every $\lambda\in[0,1/2]$, the elementary inequality $\log(1+x)\geq x-x^2$ for $x\geq-1/2$ gives
\[
  \sum_{t=1}^T\log(1+\lambda v_t)
  \geq T\lambda\bar v_T-\lambda^2\sum_{t=1}^Tv_t^2.
\]
If $\bar v_T>0$, set $Q_T=T^{-1}\sum_tv_t^2$ and choose $\lambda=\min\{\bar v_T/(2Q_T),1/2\}$, with the convention that the first term is $+\infty$ when $Q_T=0$.  If $\bar v_T\leq Q_T$, the lower bound is $T\bar v_T^2/(4Q_T)\geq T\bar v_T^2/4$.  If $\bar v_T>Q_T$, the choice $\lambda=1/2$ yields at least $T\bar v_T/4\geq T\bar v_T^2/4$.  When $\bar v_T\leq0$, the choice $\lambda=0$ yields zero.  Therefore
\begin{equation}
  \max_{0\leq\lambda\leq1/2}\sum_{t=1}^T\log(1+\lambda v_t)
  \geq \frac{T}{4}\pos{\bar v_T}^{2}.
  \label{eq:constant-bet-bound}
\end{equation}
Combining \cref{eq:constant-bet-bound} with logarithmic regret relative to the best constant bet proves the wealth inequality used in \cref{thm:gap-monitor}.

\subsection{Normalization for the two examples}

For bounded means, $r_t=(X_t-Y_t)/2$ and $w_t=2u_t$, where the betting paper uses $\norm{u_t}_2\leq1/2$.  Thus $\ip{w_t}{r_t}=\ip{u_t}{X_t-Y_t}$ and $\dist(\E r_t,0)=\norm{\E X-\E Y}_2/2$.

For a kernel $k$ with $k(x,x)\leq1$, let $\phi(x)=k(x,\cdot)$, $r_t=(\phi(X_t)-\phi(Y_t))/2$, and $w_t=2g_t$.  The paper's constraint $\norm{g_t}_{\cH_k}\leq1/2$ is equivalent to $\norm{w_t}_{\cH_k}\leq1$, and
\[
  \ip{w_t}{r_t}_{\cH_k}=g_t(X_t)-g_t(Y_t),\qquad
  \norm{\E r_t}_{\cH_k}=\frac12\MMD_k(P_X,P_Y).
\]
The factors of two affect constants but not consistency or the orders in the stopping-time bounds.

\end{document}